\documentclass[a4paper]{article}
\usepackage[T1]{fontenc}
\usepackage[utf8]{inputenc}
\usepackage{authblk}
\usepackage{amssymb}
\usepackage[sort&compress,numbers,square]{natbib}
\usepackage[titlenumbered,ruled,lined,linesnumbered,algo2e]{algorithm2e}
\usepackage{graphicx}
\usepackage{amsmath}
\usepackage{xspace}
\usepackage{bm}
\usepackage{amsthm}
\usepackage{afterpage}
\usepackage{indentfirst}
\usepackage{enumerate}
\usepackage{url}
\usepackage{comment}
\usepackage[scale={0.74,0.78}, hmarginratio=1:1, vmarginratio=1:1, headheight=2ex, headsep = 2ex, footskip = 3.6ex, nomarginpar]{geometry}
\usepackage{color}
\usepackage{hyperref}

\newtheorem{theorem}{Theorem}
\newtheorem{problem}[theorem]{Problem}
\newtheorem{lemma}[theorem]{Lemma}
\newtheorem{proposition}[theorem]{Proposition}
\newtheorem{definition}[theorem]{Definition}

\newtheorem{corollary}[theorem]{Corollary}
\newtheorem{remark}[theorem]{Remark}

\newcommand{\nbb}{\ensuremath{\mathbb{N}}}
\newcommand{\hcal}{\ensuremath{\mathcal{H}}}

\newcommand{\xcal}{\mathcal{X}}
\newcommand{\zcal}{\mathcal{Z}}
\newcommand{\ycal}{\mathcal{Y}}

\newcommand{\ebb}{\mathbb{E}}
\newcommand{\bw}{\mathbf{w}}
\newcommand{\bz}{\mathbf{z}}
\newcommand{\rbb}{\mathbb{R}}

\title{Multi-class SVMs: From Tighter Data-Dependent Generalization Bounds to Novel Algorithms}
\author[1]{Yunwen Lei\thanks{yunwen.lei@hotmail.com}}
\author[2]{\"Ur\"un Dogan\thanks{urundogan@gmail.com}}
\author[3,4]{Alexander Binder\thanks{alexander.binder@tu-berlin.de}}
\author[5]{Marius Kloft\thanks{kloft@hu-berlin.de}}
\affil[1]{Department of Mathematics, City University of Hong Kong}
\affil[2]{Microsoft Research, Cambridge CB1 2FB, UK}
\affil[3]{Machine Learning Group, TU Berlin}
\affil[4]{ISTD pillar, Singapore University of Technology and Design}
\affil[5]{Department of Computer Science, Humboldt University of Berlin}

\begin{document}
\maketitle

\begin{abstract}
This paper studies the generalization performance of multi-class classification algorithms,
for which we obtain---for the first time---a data-dependent generalization error bound with a \emph{logarithmic} dependence on the class size,
substantially improving the state-of-the-art linear dependence in the existing data-dependent generalization analysis.
The theoretical analysis motivates us to introduce a new multi-class classification machine based on $\ell_p$-norm regularization,
where the parameter $p$ controls the complexity of the corresponding bounds.
We derive an efficient optimization algorithm based on Fenchel duality theory.
Benchmarks on several real-world datasets show that the proposed algorithm can achieve significant accuracy gains over the state of the art.
\end{abstract}

\section{Introduction}

Typical multi-class application domains such as natural language processing~\citep{zhang2004class}, information retrieval~\citep{hofmann2003learning},
image annotation~\citep{deng2009imagenet} and web advertising~\citep{beygelzimer2009conditional} involve tens or hundreds of thousands of classes, and yet these datasets are still growing~\citep{bengio2010label}.
To handle such learning tasks, it is essential to build algorithms that scale favorably with respect to the number of classes.
Over the past years, much progress in this respect has been achieved on the algorithmic side \citep{jain2009active,beygelzimer2009conditional,dekel2010multiclass,bengio2010label},
including efficient stochastic gradient optimization strategies \citep{gupta2014training}.

Although also theoretical properties such as consistency \citep{zhang2004statistical,tewari2007consistency,glasmachers2010universal} and
finite-sample behavior \citep{zhang2004class,mohri2012foundations,kuznetsov2014multi,koltchinskii2002empirical,guermeur2002combining}
have been studied, there still is a discrepancy between algorithms and theory in the sense that the corresponding theoretical bounds do often not scale well with respect to the number of classes.
This discrepancy occurs the most strongly in research on \emph{data-dependent} generalization bounds,
that is, bounds that can measure generalization performance of prediction models purely from the training samples,
and which thus are very appealing in model selection \citep{oneto2011impact}.
A crucial advantage of these bounds is that they can better capture the properties of the distribution that has generated the data,
which can lead to tighter estimates \citep{koltchinskii2000rademacher} than conservative data-independent bounds.

To our best knowledge, for multi-class classification, the first data-dependent error bounds were given by \cite{koltchinskii2002empirical}.
These bounds exhibit a quadratic dependence on the class size
and were used by \cite{mohri2012foundations} and \cite{cortes2013multi} to
derive bounds for kernel-based multi-class classification and multiple kernel learning problems, respectively.
More recently, \cite{kuznetsov2014multi} improve the quadratic dependence to a linear dependence by introducing a novel surrogate for the multi-class margin that is
independent on the true realization of the class label.

However, a heavy dependence on the class size, such as linear or quadratic, implies a poor generalization guarantee for large-scale multi-class classification problems with a massive number of classes.
In this paper, we show data-dependent generalization bounds for multi-class classification problems that---for the first time---exhibit a \emph{sublinear} dependence on the number of classes.
Choosing appropriate regularization, this dependence can be as mild as logarithmic.
We achieve these improved bounds via the use of Gaussian complexities, while previous bounds are based on a well-known structural result on Rademacher complexities for classes induced by the maximum operator.
The proposed proof technique based on Gaussian complexities exploits potential coupling among different components of the multi-class classifier,
while this fact is ignored by previous analyses.


The result shows that the generalization ability is strongly impacted by the employed regularization.
Which motivates us to propose a new learning machine performing block-norm regularization over the multi-class components.
As a natural choice we investigate here the application of the proven $\ell_p$ norm \citep{kloft2011lp}.
This results in a novel $\ell_p$-norm multi-class support vector machine (SVM),
which contains the classical model by Crammer \& Singer \cite{crammer2002algorithmic} as a special case for $p=2$.
The bounds indicate that the parameter $p$ crucially controls the complexity of the resulting prediction models.

We develop an efficient optimization algorithm for the proposed method based on its Fenchel dual representation.
We empirically evaluate its effectiveness on several standard benchmarks for multi-class classification taken from various domains,
where the proposed approach significantly outperforms the state-of-the-art method of \cite{crammer2002algorithmic} by up to $1\%$.

The remainder of this paper is structured as follows.
Section \ref{sec:main-results} introduces the problem setting and presents the main theoretical results.
Motivated by which we propose a new multi-class classification model in Section \ref{sec:algorithm}
and give an efficient optimization algorithm based on Fenchel duality theory.
In Section \ref{sec:empirical} we evaluate the approach for the application of visual image recognition
and on several standard benchmark datasets taken from various application domains. Section \ref{sec:conclusion} concludes.

\section{Theory\label{sec:main-results}}

\subsection{Problem Setting}

This paper considers multi-class classification problems with $c\geq2$  classes. Let $\xcal$ denote the input space and $\ycal=\{1,2,\ldots,c\}$ denote the
output space. Assume that we are given a sequence of examples $S=\{(x_1,y_1),\ldots,(x_n,y_n)\}\in(\xcal\times\ycal)^n$, independently drawn according to a probability measure $P$ defined on the sample space
$\zcal=\xcal\times\ycal$. Based on the training examples $S$, we wish to learn a prediction rule $h_{\bz}$ from a space $H$ of hypothesis mapping from $\zcal$ to
$\rbb$ and use the mapping $x\to\arg\max_{y\in\ycal}h_{\bz}(x,y)$ to predict. For any hypothesis $h\in H$, the margin $\rho_h(x,y)$ of the function $h$ at a labeled example $(x,y)$ is
$\rho_h(x,y)\to h(x,y)-\max_{y^{'}\neq y}h(x,y^{'}).$
The prediction rule $h$ makes an error at $(x,y)$ if $\rho_h(x,y)\leq0$ and thus the expected risk incurred from using $h$ for prediction is $R(h):=\ebb[1_{\rho_h(x,y)\leq0}].$

\subsection{Notation}

Any function $h:\xcal\times\ycal\to\rbb$ can be equivalently represented by the function vector $(h_1,\ldots,h_c)$ with $h_j(x)=h(x,j),\forall j=1,\ldots,c$. We denote by $\widetilde{H}:=\{\rho_h(x,y):h\in H\}$ the class of margin functions associated to $H$. Let $k:\xcal\times\xcal\to\rbb$ be a mercer kernel with $\phi(x)$ being the associated feature map, i.e., $k(x,y)=\langle\phi(x),\phi(y)\rangle$.  We denote by $\|\cdot\|_*$ the dual norm of $\|\cdot\|$, i.e., $\|w\|_*=\sup_{\|\bar{w}\|\leq1}\langle w,\bar{w}\rangle$. For a convex function $f$, we denote by $f^*$ its Fenchel conjugate, i.e., $f^*(v):=\sup_w[\langle w,v\rangle-f(w)].$ For any $\bw=(\bw_1,\ldots,\bw_c)$ we define the $\ell_{2,p}$-norm by $\|\bw\|_{2,p}=[\sum_{j=1}^c\|\bw_j\|_2^p]^{1/p}$. For any $p\geq1$, we denote by $p^*$ the dual exponent of $p$ satisfying $1/p+1/p^*=1$ and $\bar{p}:=p(2-p)^{-1}$. 
In the remainder of the paper, we require the following definitions.

\begin{definition}[Strong Convexity]
  A function $f:\xcal\to\rbb$ is said to be $\beta$-strongly convex w.r.t. a norm $\|\cdot\|$ iff $\forall x,y\in\xcal$ and $\forall \alpha\in(0,1)$, we have
  $$f(\alpha x+(1-\alpha) y)\leq\alpha f(x)+(1-\alpha)f(y)-\frac{\beta}{2}\alpha(1-\alpha)\|x-y\|^2.$$
\end{definition}

\begin{definition}[Regular Loss]
  We call $\ell$ a $L$-regular loss if it satisfies the following properties:
  \begin{enumerate}[(i)]\setlength{\itemsep}{-3pt}
    \item $\ell(t)$ bounds the $0$-$1$ loss from above: $\ell(t)\geq 1_{t\leq0}$;
    \item $\ell$ is $L$-Lipschitz in the sense $|\ell(t_1)-\ell(t_2)|\leq L|t_1-t_2|$;
    \item $\ell(t)$ is decreasing and it has a zero point $c_\ell$, i.e., $\ell(c_\ell)=0$.
  \end{enumerate}
\end{definition}
Some examples of $L$-regular loss functions include the hinge $\ell_h(t)=(1-t)_+$ and the margin loss
\begin{equation}\label{margin-loss}
  \ell_\rho(t)=1_{t\leq0}+(1-t\rho^{-1})1_{0<t\leq\rho},\quad\rho>0.
\end{equation}

\subsection{Main results}
Our discussion on data-dependent generalization error bounds is based on the established methodology of Rademacher and Gaussian complexities~\citep{bartlett2002rademacher}.
\begin{definition}[Rademacher and Gaussian Complexity]
  Let $H$ be a family of real-valued functions defined on $\mathcal{Z}$ and $S=(z_1,\ldots,z_n)$ a fixed sample of size $n$ with elements in $\mathcal{Z}$. Then, the empirical Rademacher and Gaussian complexities of $H$ with respect to the sample $S$ are defined by
  $$
    \mathfrak{R}_S(H)=\ebb_{\bm{\sigma}}\big[\sup_{h\in H}\frac{1}{n}\sum_{i=1}^n\sigma_ih(z_i)\big],\quad \mathfrak{G}_S(H)=\ebb_{\bm{g}}\big[\sup_{h\in H}\frac{1}{n}\sum_{i=1}^ng_ih(z_i)\big],
    $$
  where $\sigma_1,\ldots,\sigma_n$ are independent random variables with equal probability taking values $+1$ or $-1$, and $g_1,\ldots,g_n$ are independent $N(0,1)$ random variables.
\end{definition}
Note that we have the following comparison inequality relating Rademacher and Gaussian complexities~\citep{ledoux1991probability}:
\vspace{-0.1cm}
\begin{equation}\label{gaussian-rademacher}
  \mathfrak{R}_S(H)\leq\sqrt{\frac{\pi}{2}}\mathfrak{G}_S(H)\leq 3\sqrt{\frac{\pi}{2}}\sqrt{\log n}\mathfrak{R}_S(H).
\end{equation}
Existing work on data-dependent generalization bounds for multi-class classifiers \citep{mohri2012foundations,kuznetsov2014multi,cortes2013multi,koltchinskii2002empirical} build on the following structural result on Rademacher complexities (e.g., \cite{mohri2012foundations}, Lemma 8.1):
\begin{equation}\label{Rademacher-maximum-lem-8-1}
  \vspace{-0.05cm}
  \mathfrak{R}_S(\max\{h_1,\ldots,h_c\}:h_j\in H_j,j=1,\ldots,c)\leq\sum_{j=1}^c\mathfrak{R}_S(H_j),
	\vspace{-0.05cm}
\end{equation}
where $H_1,\ldots,H_c$ are $c$ hypothesis sets. This result is crucial for the standard generalization analysis of multi-class classification
since the definition of margin involves the maximum operator, which is removed by the above lemma, but at the expense of a linear dependency on the number of classes.
In the following we show that this linear dependency is suboptimal because \eqref{Rademacher-maximum-lem-8-1} does not take into account the coupling among different classes.
For example, a common regularizer used in multi-class classification algorithms is $r(h)=\sum_{j=1}^c\|h_j\|_2^2$ \citep{crammer2002algorithmic},
for which the components $h_1,\ldots,h_c$ are correlated via a $\|\cdot\|_{2,2}$ regularizer,
and the bound Eq. \eqref{Rademacher-maximum-lem-8-1} ignoring this correlation would not be effective in this case \citep{koltchinskii2002empirical,mohri2012foundations,cortes2013multi,kuznetsov2014multi}.

As a remedy, we here introduce a new structural complexity result on function classes induced by general classes via the maximum operator,
while allowing to preserve the correlations among different components meanwhile.
Instead of considering the Rademacher complexity, Lemma \ref{lem:rademacher-max} concerns the structural relationship on Gaussian complexities
since it is based on a comparison result among different Gaussian processes.

\begin{lemma}[Structural result on Gaussian complexity]\label{lem:rademacher-max}
  Let $H$ be a class of functions defined on $\xcal\times\ycal$ with $\ycal=\{1,\ldots,c\}$. Let $g_1,\ldots,g_{nc}$ be independent $N(0,1)$ distributed random variables. Then, for any sample $S=\{x_1,\ldots,x_n\}$ of size $n$, we have
	\vspace{-0.05cm}
  \begin{equation}\label{gaussian-max}
    \mathfrak{G}_S\big(\{\max\{h_1,\ldots,h_c\}:h\in H\}\big)\leq\frac{1}{n}\ebb_{\bm{g}}\sup_{h\in H}\sum_{i=1}^n\sum_{j=1}^cg_{(j-1)n+i}h_j(x_i),
		\vspace{-0.05cm}
  \end{equation}
  where $\ebb_{\bm{g}}$ denotes the expectation w.r.t. to the Gaussian variables $g_1,\ldots,g_{nc}$.
\end{lemma}

The proof of Lemma \ref{thm:risk-bound-mohri} is given in Supplementary Material \ref{supp:complexity}. Equipped with Lemma \ref{lem:rademacher-max}, we are now able to present a general data-dependent margin-based generalization bound.
The proof of the following results (Theorem \ref{thm:risk-bound-mohri}, Theorem \ref{thm:risk-bounds-strong-convex} and Corollary \ref{cor:generalization-lp-regularizer}) is given in Supplementary Material \ref{supp:genbound}.
\begin{theorem}[Data-dependent generalization bound for multi-class classification]\label{thm:risk-bound-mohri}
  Let $H\subset\rbb^{\xcal\times\ycal}$ be a hypothesis class with $\ycal=\{1,\ldots,c\}$. Let $\ell$ be a $L$-regular loss function and denote $B_\ell:=\sup_{(x,y),h}\ell(\rho_h(x,y))$.
  Suppose that the examples $S=\{(x_1,y_1),\ldots,(x_n,y_n)\}$ are independently drawn from a probability measure defined on
  $\xcal\times\ycal$. Then, for any $\delta>0$, with probability at least $1-\delta$, the following multi-class classification generalization bound holds for any $h\in H$:
  $$R(h)\leq\frac{1}{n}\sum_{i=1}^n\ell(\rho_{h}(x_i,y_i))+\frac{2L\sqrt{2\pi}}{n}\ebb_{\bm{g}}\sup_{h\in H}\sum_{i=1}^n\sum_{j=1}^cg_{(j-1)n+i}h_j(x_i)+3B_\ell\sqrt{\frac{\log\frac{2}{\delta}}{2n}},$$where $g_1,\ldots,g_{nc}$ are independent $N(0,1)$ distributed random variables.
\end{theorem}
\begin{remark}
\rm
  Under the same condition of Theorem \ref{thm:risk-bound-mohri}, \cite{mohri2012foundations} derive the following data-dependent generalization bound:
  $$R(h)\leq\frac{1}{n}\sum_{i=1}^n\ell(\rho_{h}(x_i,y_i))+\frac{4Lc}{n}\mathfrak{R}_S(\Pi_1(H))+3B_\ell\sqrt{\frac{\log\frac{2}{\delta}}{2n}},$$where $\Pi_1(H):=\{x\to h(x,y):y\in\ycal,h\in H\}$. This linear dependence on $c$ is due to the use of Eq. \eqref{Rademacher-maximum-lem-8-1}.
	For comparison, Theorem \ref{thm:risk-bound-mohri} implies that the dependence on the class size is governed by the term $\sum_{i=1}^n\sum_{j=1}^cg_{(j-1)n+i}h_j(x_i)$, an advantage of which is that the components $h_1,\ldots,h_c$ are jointly coupled.
	As we will see, this allows us to derive an improved result having a favorable dependence on $c$, when a constraint is imposed on $(h_1,\ldots,h_c)$.\hfill\qed
\end{remark}
The following Theorem \ref{thm:risk-bounds-strong-convex} applies the general result in Theorem \ref{thm:risk-bound-mohri} to kernel-based methods. The hypothesis space is defined by imposing a constraint with a general strongly convex function.

\begin{theorem}[Data-dependent generalization bound for kernel-based multi-class learning algorithms]\label{thm:risk-bounds-strong-convex}
  Suppose that the hypothesis space is defined by $$H:=H_{f,\Lambda}=\{h^{\bw}=(\langle \bw_1,\phi(x)\rangle,\ldots,\langle \bw_c,\phi(x)\rangle):f(\bw)\leq\Lambda\},$$where $f$ is a $\beta$-strongly convex function w.r.t. a norm $\|\cdot\|$ defined on $H$ satisfying $f^*(0)=0$. Let $\ell$ be a $L$-regular loss function and denote $B_\ell:=\sup_{(x,y),h}\ell(\rho_h(x,y))$. Let $g_1,\ldots,g_{nc}$ be independent $N(0,1)$ distributed random variables. Then,
  for any $\delta>0$, with probability at least $1-\delta$ we have
  $$\hspace*{-0.3cm}R(h^{\bw})\leq\frac{1}{n}\sum_{i=1}^n\ell(\rho_{h^{\bw}}(x_i,y_i))+    \frac{4L}{n}\sqrt{\frac{\pi\Lambda}{\beta}\ebb_{\bm{g}}\sum_{i=1}^n\|(g_i\phi(x_i),g_{n+i}\phi(x_i),\ldots,g_{(c-1)n+i}\phi(x_i))\|^2_*}+3B_\ell\sqrt{\frac{\log\frac{2}{\delta}}{2n}}.
  $$
\end{theorem}


We now consider the following specific hypothesis spaces using a $\|\cdot\|_{2,p}$ constraint:
\begin{equation}\label{hypothesis-space-lp-regularizer}
  H_{p,\Lambda}:=\{h^{\bw}=(\langle \bw_1,\phi(x)\rangle,\ldots,\langle \bw_c,\phi(x)\rangle):\|\bw\|_{2,p}\leq\Lambda\},\quad 1\leq p\leq2.
\end{equation}
\begin{corollary}[$\ell_p$-norm multi-class SVM generalization bound]\label{cor:generalization-lp-regularizer}
  Let $\ell$ be a $L$-regular loss function and denote $B_\ell:=\sup_{(x,y),h}\ell(\rho_h(x,y))$. Then, with probability at least $1-\delta$, for any $h^{\bw}\in H_{p,\Lambda}$ the generalization error $R(h^{\bw})$ can be upper bounded by:
  $$
  \hspace*{-0.1cm}\frac{1}{n}\sum_{i=1}^n\ell(\rho_{h^{\bw}}(x_i,y_i))+3B_\ell\sqrt{\frac{\log\frac{2}{\delta}}{2n}}+
  \frac{2L\Lambda}{n}\sqrt{\sum_{i=1}^nk(x_i,x_i)}\times\begin{cases}
    \sqrt{e}(4\log c)^{1+\frac{1}{2\log c}},&\text{if }p\leq\frac{2\log c}{2\log c - 1},\\
    \big(\frac{2p}{p-1}\big)^{2-\frac{1}{p}}c^{\frac{p-1}{p}},&\text{otherwise}.
  \end{cases}$$
\end{corollary}
\begin{remark}
\rm
  The bounds in Corollary \ref{cor:generalization-lp-regularizer} enjoy a mild dependence on the number of classes. The dependence is polynomial with exponent $\frac{p-1}{p}$ for $\frac{2\log c}{2\log c-1}< p\leq2$ and becomes logarithmic if $1\leq p\leq\frac{2\log c}{2\log c - 1}$. Which is substantially milder than the quadratic dependence established in \citep{koltchinskii2002empirical,cortes2013multi,mohri2012foundations} and the linear dependence established in \cite{kuznetsov2014multi}. Our generalization bound is data-dependent and shows clearly how the margin would affect the generalization performance (when $\ell$ is the margin loss $\ell_\rho$): a large margin $\rho$ would increase the empirical error while decrease the model's complexity, and vice versa.\hfill\qed
\end{remark}
\subsection{Comparison of the Achieved Bounds to the State of the Art}

\textbf{Related work on data-independent bounds}.
The large body of theoretical work on multi-class learning considers data-independent bounds.
Based on the $\ell_\infty$-covering number bound of linear operators, \cite{guermeur2002combining} obtain a generalization bound exhibiting a linear dependence on the class size,
which is improved by \cite{zhang2004statistical} to a radical dependence of the form $O(n^{-\frac{1}{2}}(\log^{\frac{3}{2}}n)\frac{\sqrt{c}}{\rho})$.
Under conditions analogous to Corollary \ref{cor:generalization-lp-regularizer}, \cite{hill2007framework} derive a class-size independent generalization guarantee.
However, their bound is based on a delicate definition of margin, which is why it is commonly not used in the mainstream multi-class literature.
\cite{zhang2004class} derive the following generalization bound
\begin{multline}\label{bound-zhang-tong-1}
  \ebb\Big[\frac{1}{p}\log\Big(1+\sum_{\tilde{y}\neq y}e^{p(\rho-\langle\hat{\bw}_y-\hat{\bw}_{\tilde{y}},\phi(x)\rangle)}\Big)\Big]\leq\inf_{\bw\in H}\Big[\frac{1}{p}\log\Big(1+\sum_{\tilde{y}\neq y}e^{p(\rho-\langle\bw_y-\bw_{\tilde{y}},\phi(x)\rangle)}\Big)\\
  +\frac{\lambda n}{2(n+1)}\|\bw\|_{2,2}^2\Big]+\frac{2\sup_{x\in\xcal}k(x,x)}{\lambda n},
\end{multline}
where $\rho$ is a margin condition, $p>0$ a scaling factor, and $\lambda$ a regularization parameter.
Eq. \eqref{bound-zhang-tong-1} is class-size independent, yet Corollary \ref{cor:generalization-lp-regularizer} shows superiority in the following aspects:
first, for SVMs (i.e., margin loss $\ell_\rho$), our bound consists of an empirical error ($\frac{1}{n}\sum_{i=1}^n\ell_\rho(\rho_{h^{\bw}}(x_i,y_i))$) and a complexity term
divided by the margin value (note that $L=1/\rho$ in Corollary \ref{cor:generalization-lp-regularizer}).
When the margin is large (which is often desirable)~\citep{koltchinskii2002empirical},
the last term in the bound given by Corollary \ref{cor:generalization-lp-regularizer} becomes small,
while---on the contrary----the bound \eqref{bound-zhang-tong-1} is an increasing function of $\rho$, which is undesirable.
Secondly, Theorem \ref{thm:risk-bounds-strong-convex} applies to general loss functions,
expressed through a strongly convex function over a general hypothesis space, while the bound \eqref{bound-zhang-tong-1} only applies to a specific regularization algorithm.
Lastly, all the above mentioned results are conservative data-independent estimates.


\textbf{Related work on data-dependent bounds}.
The techniques used in above mentioned papers do not straightforward translate to data-dependent bounds,
which is the type of bounds in the focus of the present work.
The investigation of these was initiated, to our best knowledge, by \cite{koltchinskii2002empirical}:  
with the structural complexity bound \eqref{Rademacher-maximum-lem-8-1} for function classes induced via the maximal operator,
\cite{koltchinskii2002empirical} derive a margin bound admitting a quadratic dependency on the number of classes.
\cite{mohri2012foundations} use these results in \cite{koltchinskii2002empirical} to study the generalization performance of multi-class SVMs,
where the components $h_1,\ldots,h_c$ are coupled with an $\|\cdot\|_{2,p},p\geq1$ constraint.
Due to the usage of the suboptimal Eq. \eqref{Rademacher-maximum-lem-8-1},
\cite{mohri2012foundations} obtain a margin bound growing quadratically w.r.t. the number of classes.
\cite{cortes2013multi} develop a new multi-class classification algorithm based on a natural notion called the multi-class margin of a kernel.
\cite{cortes2013multi} also present a novel multi-class Rademacher complexity margin bound based on Eq. \eqref{Rademacher-maximum-lem-8-1},
and the bound also depends quadratically on the class size.
More recently, \cite{kuznetsov2014multi} give a refined Rademacher complexity bound for multi-class classification with a linear dependence on the class size.
The key reason for  this improvement is the introduction of $\rho_{\theta,h}:=\min_{y^{'}\in\ycal}[h(x,y)-h(x,y^{'})+\theta1_{y^{'}= y}]$ bounding margin $\rho_h$ from below,
and since the maximum operation in $\rho_{\theta,h}$ is applied to the set $\ycal$ rather than the subset $\ycal-\{y_i\}$ for $\rho_h$, one needs not to consider the random realization of $y_i$.
We also use this trick in our proof of Theorem \ref{thm:risk-bound-mohri}. However, \cite{kuznetsov2014multi} failed to improve this linear dependence to a logarithmic dependence,
as we achieved in Corollary \ref{cor:generalization-lp-regularizer}, due to the use of the suboptimal structural result \eqref{Rademacher-maximum-lem-8-1}.


\section{Algorithms\label{sec:algorithm}}

Motivated by the generalization analysis given in Section \ref{sec:main-results}, we now present a new multi-class learning algorithm,
based on performing empirical risk minimization in the hypothesis space \eqref{hypothesis-space-lp-regularizer}.
This corresponds to the following $\ell_p$-norm multi-class SVM ($p\geq1$):
\begin{problem}[Primal problem: $\ell_p$-norm multi-class SVM]\label{prob:primal-problem}
\begin{equation}\label{primal-problem}
\tag{P}
\begin{split}
  \min_{\bw}&\;\frac{1}{2}\Big[\sum_{j=1}^c\|\bw_j\|_2^p\Big]^{\frac{2}{p}}+C\sum_{i=1}^n\ell(t_i),\\
  \text{s.t.}&\;t_i=\langle \bw_{y_i},\phi(x_i)\rangle-\max_{y\neq y_i}\langle \bw_y,\phi(x_i)\rangle,
\end{split}
\end{equation}
\end{problem}
For $p=2$ we recover the seminal multi-class algorithm by Crammer \& Singer \citep{crammer2002algorithmic},
which is thus a special case of the proposed formulation.
An advantage of the proposed approach over \cite{crammer2002algorithmic} can be that, as shown in Corollary \ref{cor:generalization-lp-regularizer},
the dependence of the generalization performance on the class size becomes milder as $p$ decreases to $1$.

\subsection{Dual problems}
Since the optimization problem \eqref{primal-problem} is convex, we can derive the associated dual problem for the construction of efficient optimization algorithms.
The derivation of the following dual problem is deferred to Supplementary Material \ref{supp:dual}.
For a matrix $\bm{\alpha}\in\rbb^{n\times c}$, we denote by $\bm{\alpha}_i$ the $i$th row. Denote by $\bm{e}_j$ the $j$-th unit vector in $\rbb^c$ and $\bm{1}$ the vector in $\rbb^c$ with all components being zero.
\begin{problem}[Completely dualized problem for general loss functions\label{prop:complete-dual}]
  The Lagrangian dual problem of \eqref{prob:primal-problem} is:
  \begin{equation}\label{dual-problem}
  \tag{D}
    \begin{split}
    \sup_{\bm{\alpha}\in\rbb^{n\times c}}&-\frac{1}{2}\Big[\sum_{j=1}^c\big\|\sum_{i=1}^n\alpha_{ij}\phi(x_i)\big\|_2^{\frac{p}{p-1}}\Big]^{\frac{2(p-1)}{p}}-C\sum_{i=1}^n\ell^*(-\frac{\alpha_{iy_i}}{C})\\
    \text{s.t.}&\;\alpha_{ij}\leq0\;\land\;\bm{\alpha}_i\cdot\bm{1}=0,\quad\forall j\neq y_i, i=1,\ldots,n.
    \end{split}
  \end{equation}
\end{problem}

\begin{theorem}[\textsc{Representer theorem}\label{thm:repre}]
  For any dual variable $\bm{\alpha}\in\rbb^{n\times c}$, the associated primal variable $\bw=(\bw_1,\ldots,\bw_c)$ minimizing the Lagrangian saddle problem can be represented by:
  $$\bw_j=\big[\sum_{\tilde{j}=1}^c\|\sum_{i=1}^n\alpha_{i\tilde{j}}\phi(x_i)\|_2^{p^*}\big]^{\frac{2}{p^*}-1}\big\|\sum_{i=1}^n\alpha_{ij}\phi(x_i)\big\|_2^{p^*-2}\big[\sum_{i=1}^n\alpha_{ij}\phi(x_i)\big].$$
\end{theorem}
For the hinge loss $\ell_h(t)=(1-t)_+$, we know its Fenchel-Legendre conjugate is $\ell_h^*(t)=t$ if $-1\leq t\leq 0$ and $\infty$ elsewise. Hence $\ell_h^*(-\frac{\alpha_{iy_i}}{C})=-\frac{\alpha_{iy_i}}{C}$
if $-1\leq-\frac{\alpha_{iy_i}}{C}\leq 0$ and $\infty$ elsewise. Now we have the following dual problem for the hinge loss function:
\begin{problem}[Completely dualized problem for the hinge loss ($\ell_p$-norm multi-class SVM)]\label{prob:dual-problem-hinge}
\begin{equation}\label{dual-problem-hinge}
\begin{split}
\sup_{\bm{\alpha}\in\rbb^{n\times c}}&\;-\frac{1}{2}\Big[\sum_{j=1}^c\big\|\sum_{i=1}^n\alpha_{ij}\phi(x_i)\big\|_2^{\frac{p}{p-1}}\Big]^{\frac{2(p-1)}{p}}+\sum_{i=1}^n\alpha_{iy_i}\\
\text{s.t.}&\;\bm{\alpha}_i\leq \bm{e}_{y_i}\cdot C\;\land\;\bm{\alpha}_i\cdot\bm{1}=0,\quad\forall i=1,\ldots,n.
\end{split}
\end{equation}
\end{problem}

\subsection{Optimization Algorithms}
The dual problems \eqref{dual-problem} and \eqref{dual-problem-hinge} are not quadratic programs for $p\neq2$,
and thus generally not easy to solve.
To circumvent this difficulty, we rewrite Problem \ref{prob:primal-problem} as the following equivalent problem:
\begin{equation}\label{primal-problem-equivalent}
\begin{split}
  \min_{\bw,\bm{\beta}}&\;\sum_{j=1}^c\frac{\|\bw_j\|_2^2}{2\beta_j}+C\sum_{i=1}^n\ell(t_i)\\
  \text{s.t.}&\;t_i\leq\langle \bw_{y_i},\phi(x_i)\rangle-\langle \bw_y,\phi(x_i)\rangle,\quad y\neq y_i,i=1,\ldots,n,\\
  &\;\|\bm{\beta}\|_{\bar{p}}\leq1,\bar{p}=p(2-p)^{-1},\beta_j\geq0.
\end{split}
\end{equation}
The class weights $\beta_1,\ldots,\beta_c$ in Eq. \eqref{primal-problem-equivalent} play a similar role as the kernel weights in $\ell_p$-norm multiple kernel learning (MKL) algorithms \citep{kloft2011lp}.
The equivalence between problem \eqref{primal-problem} and Eq. \eqref{primal-problem-equivalent} follows directly from Lemma 26 in \cite{micchelli2005learning},
which shows that the optimal $\bm{\beta}=(\beta_1,\ldots,\beta_c)$ in Eq. \eqref{primal-problem-equivalent} can be explicitly represented in closed form.
Motivated by the recent work on $\ell_p$-norm MKL, we propose to solve the problem \eqref{primal-problem-equivalent} via alternately optimizing $\bw$ and $\bm{\beta}$.
As we will show, given temporarily fixed $\bm{\beta}$, the optimization of $\bw$ reduces to a standard multi-class classification problem. Furthermore, the update of $\bm{\beta}$, given fixed $\bw$, can be achieved via an analytic formula.
\begin{problem}[Partially dualized problem for a general loss\label{prop:partial-dual}]
  For fixed $\bm{\beta}$, the partial dual problem for the sub-optimization problem \eqref{primal-problem-equivalent} w.r.t. $\bw$ is
  \begin{equation}\label{dual-problem-partial}
    \begin{split}
    \sup_{\bm{\alpha}\in\rbb^{n\times c}}&-\frac{1}{2}\sum_{j=1}^c\beta_j\big\|\sum_{i=1}^n\alpha_{ij}\phi(x_i)\big\|_2^2-C\sum_{i=1}^n\ell^*(-\frac{\alpha_{iy_i}}{C})\\
    \text{s.t.}&\;\alpha_{ij}\leq0\;\land\;\bm{\alpha}_i\cdot\bm{1}=0,\quad\forall j\neq y_i, i=1,\ldots,n.
    \end{split}
  \end{equation}
 The primal variable $\bm{w}$ minimizing the associated Lagrangian saddle problem is
  \begin{equation}\label{partial-representation}
    \bw_j=\beta_j\sum_{i=1}^n\alpha_{ij}\phi(x_i).
  \end{equation}
\end{problem}

We defer the proof to Supplementary Material \ref{supp:dual}. Analogous to Problem \ref{prob:dual-problem-hinge}, we have the following partial dual problem for the hinge loss.
\begin{problem}[Partially dualized problem for the hinge loss ($\ell_p$-norm multi-class SVM)\label{prop:partial-dual-hinge}]
\begin{equation}\label{partial-dual-hinge-loss}
\begin{split}
\sup_{\bm{\alpha}\in\rbb^{n\times c}}&\;f(\bm{\alpha}):=-\frac{1}{2}\sum_{j=1}^c\beta_j\big\|\sum_{i=1}^n\alpha_{ij}\phi(x_i)\big\|_2^2+\sum_{i=1}^n\alpha_{iy_i}\\
\text{s.t.}&\;\bm{\alpha}_i\leq \bm{e}_{y_i}\cdot C\;\land\;\bm{\alpha}_i\cdot\bm{1}=0,\quad\forall i=1,\ldots,n.
\end{split}
\end{equation}
\end{problem}
The Problems \ref{prop:partial-dual} and \ref{prop:partial-dual-hinge} are quadratic,
so we can use the dual coordinate ascent algorithm \citep{keerthi2008sequential} to very efficiently solve them for the case of linear kernels. To this end, we need to compute the gradient and solve the restricted problem of optimizing only one $\mathbf{\alpha}_i,\forall i$, keeping all other dual variables fixed~\citep{keerthi2008sequential}.
The gradient of $f$ can be exactly represented by $\bw$:
\vspace{-0.1cm}
\begin{equation}\label{gradient}
  \frac{\partial f}{\partial\alpha_{ij}}=-\beta_j\sum_{\tilde{i}=1}^n\alpha_{\tilde{i}j}k(x_i,x_{\tilde{i}})+1_{y_i=j}=1_{y_i=j}-\langle\bw_j,\phi(x_i)\rangle.
\end{equation}
Suppose the additive change to be applied to the current $\mathbf{\alpha}_i$ is $\mathbf{\delta}\mathbf{\alpha}_i$, then
\begin{align*}
  &f(\alpha_1,\ldots,\alpha_{i-1},\alpha_i+\delta\alpha_i,\alpha_{i+1},\ldots,\alpha_n)\\
  &=-\sum_{j=1}^c\beta_j\sum_{\tilde{i}=1}^n\alpha_{\tilde{i}j}(\alpha_{ij}+\delta\alpha_{ij})k(x_i,x_{\tilde{i}})-\frac{1}{2}\sum_{j=1}^c\beta_j[\delta\alpha_{ij}]^2k(x_i,x_i)+\delta\alpha_{iy_i}+\text{const}\\
  &=\sum_{j=1}^c\frac{\partial f}{\partial\alpha_{ij}}\delta\alpha_{ij}-\frac{1}{2}\sum_{j=1}^c\beta_jk(x_i,x_i)[\delta\alpha_{ij}]^2+\text{const}.
\end{align*}
Therefore, the sub-problem of optimizing $\bm{\delta}\alpha_i$ is given by
\begin{equation}
  \begin{split}
    \max_{\bm{\delta\alpha_i}}&\;-\frac{1}{2}\sum_{j=1}^c\beta_jk(x_i,x_i)[\delta\alpha_{ij}]^2+\sum_{j=1}^c\frac{\partial f}{\partial\alpha_{ij}}\delta\alpha_{ij}\\
    \text{s.t.}&\;\bm{\delta\alpha}_i\leq \bm{e}_{y_i}\cdot C-\bm{\alpha}_i\;\land\;\bm{\delta\alpha}_i\cdot\bm{1}=0.
  \end{split}
\end{equation}
We now consider the subproblem of updating class weights $\bm{\beta}$ with temporarily fixed $\bm{w}$, for which we have the following analytic solution. The proof is deferred to the Supplementary Material \ref{supp:micchelli}.
\begin{proposition}(Solving the subproblem with respect to the class weights)\label{prop:mixture-update}
  Given fixed $\bw_j$, the minimal $\beta_j$ optimizing the problem \eqref{primal-problem-equivalent} is attained at
  \begin{equation}\label{class-weight-update}
    \beta_j=\|\bw_j\|_2^{2-p}\bigg(\sum_{\tilde{j}=1}^c\|\bw_{\tilde{j}}\|_2^p\bigg)^{\frac{p-2}{p}}.
  \end{equation}
\end{proposition}
The update of $\beta_j$ based on Eq. \eqref{class-weight-update} requires calculating $\|\bw_j\|_2^2$, which can be easily fulfilled by recalling the representation established in Eq. \eqref{partial-representation}.

The resulting training algorithm for the proposed $\ell_p$-norm multi-class SVM is given Algorithm \ref{algorithm:wrapper}. The algorithm alternates between solving a multi-class SVM problem for fixed class weights (Line 3) and updating the class weights in a closed-form manner (Line 5). Recall that Problem \ref{prop:complete-dual} establishes a completely dualized problem, which can be used as a sound stopping and evaluation criterion for the optimization algorithm.
\begin{algorithm2e}[htbp]\label{algorithm:wrapper}
\SetKwInOut{Input}{input}
  \caption{Training algorithm for $\ell_p$-norm multi-class classification.}
    \Input{examples $\{(x_i,y_i)_{i=1}^n\}$ and the kernel $k$.}
    \BlankLine
    initialize $\beta_{j}=\sqrt[\bar{p}]{1/c}, \bw_j=0$ for all $j=1,\ldots,c$\\
    \While{Optimality conditions are not satisfied}{
    optimize the multi-class classification problem \eqref{dual-problem-partial}\\
    compute $\|\bw_j\|_2^2$ for all $j=1,\ldots,c,$ according to Eq. \eqref{partial-representation}\\
    update $\beta_j$ for all $j=1,\ldots,c,$ according to Eq. \eqref{class-weight-update}\\
    }
\end{algorithm2e}

\section{Empirical Analysis\label{sec:empirical}}

We implemented the proposed $\ell_p$-norm multi-class SVM algorithm (Algorithm~\ref{algorithm:wrapper}) in C++ and solved the involved MC-SVM problem using dual coordinate ascent  \cite{keerthi2008sequential}.
We experiment on three benchmark datasets: the Sector dataset studied in \cite{rennie2001improving}, the News 20 dataset collected and originally used for text classification by \cite{lang1995newsweeder},
and the Rcv1 dataset collected by \cite{lewis2004rcv1}. 
Table \ref{tab:data_set} gives a description of these datasets.

\begin{table*}[!h]
\small
\setlength{\tabcolsep}{3pt}
\centering
  \begin{tabular}{*{5}{|c}|}\hline
  Dataset & No. of Classes & No. of Training Examples & No. of Test Examples & No. of Attributes \\\hline
  Sector & $105$ & $6,412$ & $3,207$ & $55,197$ \\\hline
  News 20 & $20$ & $15,935$ & $3,993$ & $62,060$ \\\hline
  Rcv1 & $53$ & $15,564$ & $518,571$ & $47,236$ \\\hline
  \end{tabular}
    \caption{Description of datasets used in the experiments.
		\label{tab:data_set}}
		\vspace{-0.3cm}
\end{table*}
\begin{table*}[!h]
\vspace{0.6cm}
\small
\setlength{\tabcolsep}{3pt}
\centering
  \begin{tabular}{*{4}{|c}|}\hline
 Method / Dataset           & Sector              & News 20              &  Rcv1               \\\hline
  $\ell_p$-norm MC-SVM & $\bm{94.20\pm0.34}$ & $\bm{86.19\pm0.12}$  & $\bm{85.74\pm0.71}$ \\\hline
  Crammer \& Singer & $93.89\pm0.27$      & $85.12\pm0.29$       & $85.21\pm0.32$      \\\hline
  \end{tabular}
  \caption{Test set accuracies achieved by the classical Crammer \& Singer and the proposed $\ell_p$-norm multi-class SVM on the benchmark datasets.
		\label{tab:results}}
		\vspace{-0.3cm}
\end{table*}

We compare with the classical multi-class classification algorithm proposed by Crammer \& Singer \cite{crammer2002algorithmic}, which constitutes strong baseline for these datasets \citep{keerthi2008sequential}.
We employ a $5$-fold cross validation on the training set to tune the regularization parameter $C$ by grid search over the set $\{2^{-12},2^{-11},\ldots,2^{12}\}$
and the parameter $p$ from the interval $[1.2,1.25,\ldots,10]$.
For the parameter $p$ we first use a larger grid of step size 0.5 and then a finer grid of step size 0.1 around the optimum.
Note that the model parameters are tuned separately for each training set and only based on the training set, not the test set.
We repeat the experiments $10$ times, and report in Table \ref{tab:results} on the average accuracy and standard deviations attained on the test set.

We observe that the proposed $\ell_p$-norm MC-SVM consistently outperforms the method by Crammer \& Singer \cite{crammer2002algorithmic} on all considered datasets.
Specifically, our method attains $0.31\%$ accuracy gain on Sector, $1.07\%$ accuracy gain on News 20, and $0.53\%$ accuracy gain on Rcv1.
These promising results indicate that the proposed $\ell_p$-norm multiclass SVM could further lift the state of the art in multi-class classification,
even in real-world applications beyond the ones studied in this paper.

\section{Conclusion\label{sec:conclusion}}

Motivated by the ever growing size of multi-class datasets in real-world applications such as image annotation and web advertising,
which involve tens or hundreds of thousands of classes,
we studied the influence of the class size on the generalization behavior of multi-class classifiers.
We focus here on data-dependent generalization bounds enjoying the ability to capture the properties of the distribution that has generated the data.
Of independent interest, for hypothesis classes that are given as a maximum over base classes,
we developed a new structural result on Gaussian complexities that is able to preserve the coupling among different components,
while the existing structural results ignore this coupling and may yield suboptimal generalization bounds.
We applied the new structural result to study learning rates for multi-class classifiers,
and derived, for the first time, a data-dependent bound with a logarithmic dependence on the class size,
which substantially outperforms the linear dependence in the state-of-the-art data-dependent generalization bounds.

Motivated by the theoretical analysis, we proposed a novel $\ell_p$-norm regularized multi-class support vector machine,
where the parameter $p$ controls the complexity of the corresponding bounds.
This class of algorithms contains the classical model by Crammer \& Singer \cite{crammer2002algorithmic} as a special case for $p=2$.
We developed an effective optimization algorithm based on the Fenchel dual representation.
For several standard benchmarks for multi-class classification taken from various domains,
the proposed approach surpassed the state-of-the-art method of Crammer \& Singer \cite{crammer2002algorithmic}, by up to $1\%$.

An exciting future direction will be to derive a data-dependent bound that is completely independent of the class size (even overcoming the mild logarithmic dependence of our bounds).
To this end, we will study more powerful structural results than Lemma \ref{lem:rademacher-max} for controlling complexities of function classes induced via the maximum operator.
As a good starting point to this end, we will consider $\ell_\infty$-covering numbers.

\begin{small}
\setlength{\bibsep}{0.03cm}
\bibliographystyle{ieeetr}
\bibliography{MCC}
\end{small}

\newpage
\appendix
\section*{Supplementary Material}
\numberwithin{equation}{section}
\numberwithin{theorem}{section}
\numberwithin{figure}{section}
\numberwithin{table}{section}
\renewcommand{\thesection}{{\Alph{section}}}
\renewcommand{\thesubsection}{\Alph{section}.\arabic{subsection}}
\renewcommand{\thesubsubsection}{\Roman{section}.\arabic{subsection}.\arabic{subsubsection}}
\setcounter{secnumdepth}{-1}
\setcounter{secnumdepth}{3}

\section{Proofs on Structural Results on Gaussian Complexity~\label{supp:complexity}}
Our discussion on complexity bound is based on the following comparison results among different Gaussian processes.
\begin{lemma}[Theorem 1 in \citep{vitale2000some}\label{lem:gaussian-comparison}]
  Let $\{\mathfrak{X}_\theta:\theta\in \Theta\}$ and $\{\mathfrak{Y}_\theta:\theta\in\Theta\}$ be two non-zero real-valued Gaussian processes indexed by the same countable set $\Theta$ and suppose that
  \begin{equation}\label{increment-condition}
    \ebb[(\mathfrak{Y}_\theta-\mathfrak{Y}_{\bar{\theta}})^2]\leq \ebb[(\mathfrak{X}_\theta-\mathfrak{X}_{\bar{\theta}})^2],\quad\forall \theta,\bar{\theta}\in\Theta.
  \end{equation}
  Then, $$\ebb[\sup_\theta \mathfrak{Y}_\theta]\leq\ebb[\sup_\theta \mathfrak{X}_\theta].$$
\end{lemma}
\begin{proof}[\rm\textbf{Proof of Lemma \ref{lem:rademacher-max}}]
  Define two Gaussian processes indexed by $H$ (for any $h\in H$, we use here the equivalent representation $h=(h_1,\ldots,h_c)$):
\begin{gather*}
  \mathfrak{X}_h:=\sum_{i=1}^n\big[g_i\max\{h_1(x_i),h_2(x_i),\ldots,h_c(x_i)\}\big],\\
  \mathfrak{Y}_h:=\sum_{i=1}^n\sum_{j=1}^cg_{(j-1)n+i}h_j(x_i),\qquad\forall h\in H.
\end{gather*}
For any $h=(h_1,\ldots,h_c),\bar{h}=(\bar{h}_1,\ldots,\bar{h}_c)\in H$, the independence of the $g_i$ and the equalities $\ebb g_i^2=1$ imply that
\begin{gather*}
  \ebb[(\mathfrak{X}_h-\mathfrak{X}_{\bar{h}})^2]=\sum_{i=1}^n\big[\max\{h_1(x_i),\ldots,h_c(x_i)\}-\max\{\bar{h}_1(x_i),\ldots,\bar{h}_c(x_i)\}\big]^2\\
  \ebb[(\mathfrak{Y}_h-\mathfrak{Y}_{\bar{h}})^2]=\sum_{i=1}^n\big[(h_1(x_i)-\bar{h}_1(x_i))^2+\cdots+(h_c(x_i)-\bar{h}_c(x_i))^2\big].
\end{gather*}
For any $\bm{a}=(a_1,\ldots,a_c),\bm{b}=(b_1,\ldots,b_c)\in\rbb^c$, it can be directly checked that
$$|\max\{a_1,\ldots,a_c\}-\max\{b_1,\ldots,b_c\}|\leq\max\{|a_1-b_1|,\ldots,|a_c-b_c|\}\leq\sum_{i=1}^c|a_i-b_i|.$$
Applying the above inequality with $\bm{a}=(h_1(x_i),\ldots,h_c(x_i)),\bm{b}=(\bar{h}_1(x_i),\ldots,\bar{h}_c(x_i)),i=1,\ldots,n$, yields directly the following bounds relating the increments of the two Gaussian
 processes $\mathfrak{X}_h,\mathfrak{Y}_h$:
\begin{align*}
  \ebb[(\mathfrak{X}_h-\mathfrak{X}_{\bar{h}})^2]&\leq\sum_{i=1}^n\max\{|h_1(x_i)-\bar{h}_1(x_i)|,\ldots,|h_c(x_i)-\bar{h}_c(x_i)|\}^2\\
  &=\sum_{i=1}^n\max\{|h_1(x_i)-\bar{h}_1(x_i)|^2,\ldots,|h_c(x_i)-\bar{h}_c(x_i)|^2\}\\
  &\leq\sum_{i=1}^n\sum_{j=1}^c|h_j(x_i)-\bar{h}_j(x_i)|^2=\ebb[(\mathfrak{Y}_h-\mathfrak{Y}_{\bar{h}})^2],\quad\forall h,\bar{h}\in H.
\end{align*}
That is, the condition \eqref{increment-condition} holds and therefore Lemma \ref{lem:gaussian-comparison} can be applied here to yield the stated result.
\end{proof}

The following lemma gives a general Gaussian complexity bound for hypothesis spaces used in multi-class classification.
\begin{lemma}[Gaussian complexity of multi-class hypothesis spaces]\label{lem:gaussian-comparison-compound}
  Let $H$ be a class of functions defined on $\xcal\times\ycal$ with $\ycal=\{1,\ldots,c\}$. Let $S=\{(x_1,y_1),\ldots,(x_n,y_n)\}$ be a sequence of examples. Let $g_1,\ldots,g_{nc}$ be independent $N(0,1)$ distributed random variables.  Then the empirical Gaussian complexity of $H$ can be controlled by:
  $$\mathfrak{G}_S(H)\leq\frac{1}{n}\ebb_{\bm{g}}\sup_{h\in H}\sum_{i=1}^n\sum_{j=1}^cg_{(j-1)n+i}h_j(x_i).$$
\end{lemma}
\begin{proof}
  Define two Gaussian processes indexed by $H$:
  $$
  \mathfrak{X}_h:=\sum_{i=1}^ng_ih_{y_i}(x_i),\quad
  \mathfrak{Y}_h:=\sum_{i=1}^n\sum_{j=1}^cg_{(j-1)n+i}h_j(x_i),\quad\forall h\in H.
  $$
  For any $h,\bar{h} \in H$, it is obvious that
  \begin{align*}
    \ebb[(\mathfrak{X}_h-\mathfrak{X}_{\bar{h}})^2]&=\sum_{i=1}^n[h_{y_i}(x_i)-\bar{h}_{y_i}(x_i)]^2\\
    &\leq\sum_{i=1}^n\big[(h_1(x_i)-\bar{h}_1(x_i))^2+\cdots+(h_c(x_i)-\bar{h}_c(x_i))^2\big]\\
    &=\ebb[(\mathfrak{Y}_h-\mathfrak{Y}_{\bar{h}})^2].
  \end{align*}
  Now the stated inequality follows directly from Lemma \ref{lem:gaussian-comparison}.
\end{proof}

\section{Proofs on Generalization bounds For Multi-class Classifiers\label{supp:genbound}}
\subsection{Proof of Theorem \ref{thm:risk-bound-mohri}}
\begin{proof}[\rm\textbf{Proof of Theorem \ref{thm:risk-bound-mohri}}]
 For any $\theta>0$, introduce the following function bounding $\rho_h(x,y)$ from below:
 $$\rho_{\theta,h}(x,y)=h(x,y)-\max_{y^{'}\in\ycal}[h(x,y^{'})-\theta1_{y^{'}= y}]=\min_{y^{'}\in\ycal}[h(x,y)-h(x,y^{'})+\theta1_{y^{'}= y}].$$It can be checked that $\rho_{\theta,h}(x,y)=\min(\rho_h(x,y),\theta)$. Introduce two function
 classes derived from $\rho_{\theta,h}$:$$\widetilde{H_\theta}=\{\rho_{\theta,h}(x,y):h\in H\},\qquad\widetilde{\hcal_\theta}=\{\ell(\rho_{\theta,h}(x,y)):h\in H\}.$$According to the definition of $L$-regular loss function and
 the relationship $\rho_{\theta,h}\leq\rho_h$, we have
 $$R(h)=\ebb[1_{\rho_h(X,Y)}\leq0]\leq\ebb[1_{\rho_{\theta,h}(X,Y)}\leq0]\leq\ebb[\ell(\rho_{\theta,h}(X,Y))],$$which, together with McDiarmid inequality~\citep{mcdiarmid1989method}, yields the following inequality
 \begin{equation}\label{risk-bound-mohri-1}
    R(h)\leq\frac{1}{n}\sum_{i=1}^n\ell(\rho_{\theta,h}(x_i,y_i))+2\mathfrak{R}_S(\widetilde{\hcal_\theta})+3B_\ell\sqrt{\frac{\log\frac{2}{\delta}}{2n}},\quad\forall h\in H
 \end{equation}
 with probability at least $1-\delta$.

For the fixed parameter  $\theta=c_\ell$, we observe that $\rho_{\theta,h}(x,y)=\min(\rho_h(x,y),c_\ell)$. If $\rho_h(x,y)>c_\ell$, the definition of $L$-regular loss implies that $$\ell(\rho_{\theta,h}(x,y))=\ell(c_\ell)=0=\ell(\rho_h(x,y)).$$Otherwise, we have
$\rho_{\theta,h}(x,y)=\rho_{h}(x,y)$. Therefore, for any $(x,y)$ we have $\ell(\rho_{\theta, h}(x,y))=\ell(\rho_{h}(x,y))$, which, coupled with the Lipschitz property of $\ell$ and Eq.~\eqref{risk-bound-mohri-1}, yields the following inequality with probability at least $1-\delta$:
\begin{equation}\label{risk-bound-mohri-2}
  R(h)\leq\frac{1}{n}\sum_{i=1}^n\ell(\rho_{h}(x_i,y_i))+2L\mathfrak{R}_S(\widetilde{H_\theta})+3B_\ell\sqrt{\frac{\log\frac{2}{\delta}}{2n}},\quad \forall h\in H.
\end{equation}

The Rademacher complexity of $\widetilde{H_\theta}$ satisfies the following inequality:
\begin{equation}\label{risk-bound-mohri-3}
\begin{split}
  \mathfrak{R}_S(\widetilde{H_\theta})&=\frac{1}{n}\ebb_{\sigma}\Big[\sup_{h\in H}\sum_{i=1}^n\sigma_i\big(h(x_i,y_i)-\max_{y\in\ycal}(h(x_i,y)-\theta1_{y=y_i})\big)\Big]\\
  &\leq\frac{1}{n}\ebb_{\sigma}[\sup_{h\in H}\sum_{i=1}^n\sigma_ih(x_i,y_i)]+\frac{1}{n}\ebb_{\sigma}\Big[\sup_{h\in H}\sum_{i=1}^n\sigma_i\max_{y\in\ycal}(h(x_i,y)-\theta1_{y=y_i})\Big]\\
  &\leq\sqrt{\frac{\pi}{2}}\mathfrak{G}_S(H)+\frac{1}{n}\sqrt{\frac{\pi}{2}}\ebb_{g}\Big[\sup_{h\in H}\sum_{i=1}^ng_i\max(h_1(x_i)-\theta1_{y_i=1},\ldots,h_c(x_i)-\theta1_{y_i=c})\Big],
\end{split}
\end{equation}
where the last step follows from the relationship between Gaussian and Rademacher processes expressed in Eq. \eqref{gaussian-rademacher}.
Furthermore, according to Lemma \ref{lem:rademacher-max}, the last term of the above inequality can be addressed by
\begin{align*}
  &\ebb_{\bm{g}}[\sup_{h\in H}\sum_{i=1}^ng_i\max\{h_1(x_i)-\theta1_{y_i=1},h_2(x_i)-\theta1_{y_i=2},\ldots,h_c(x_i)-\theta1_{y_i=c}\}]\\
  &\stackrel{Lem.~\ref{lem:rademacher-max}}{\leq}\ebb_{\bm{g}}\sup_{h\in H}\sum_{i=1}^n\big[g_i(h_1(x_i)-\theta1_{y_i=1})+g_{n+i}(h_2(x_i)-\theta1_{y_i=2})+\cdots+g_{(c-1)n+i}(h_c(x_i)-\theta1_{y_i=c})\big]\\
  &=\ebb_{\bm{g}}\sup_{h\in H}\sum_{i=1}^n\big[g_ih_1(x_i)+g_{n+i}h_2(x_i)+\cdots+g_{(c-1)n+i}h_c(x_i)\big]\\
  &\qquad-\ebb_{\bm{g}}\sum_{i=1}^n[g_i\theta1_{y_i=1}+\cdots+g_{(c-1)n+i}\theta1_{y_i=c}]\\
  &=\ebb_{\bm{g}}\sup_{h\in H}\sum_{i=1}^n\big[g_ih_1(x_i)+g_{n+i}h_2(x_i)+\cdots+g_{(c-1)n+i}h_c(x_i)\big].
\end{align*}
With this inequality and using Lemma \ref{lem:gaussian-comparison-compound} to tackle $\mathfrak{G}_S(H)$, we immediately derive the following bound on $\mathfrak{R}_S(\widetilde{H_\theta})$:
$$\mathfrak{R}_S(\widetilde{H_\theta})\leq\frac{\sqrt{2\pi}}{n}\ebb_{\bm{g}}\sup_{h\in H}\sum_{i=1}^n\sum_{j=1}^cg_{(j-1)n+i}h_j(x_i).$$Putting this Rademacher complexity bound back into Eq.~\eqref{risk-bound-mohri-2},
we obtain the stated result.
\end{proof}

\subsection{Proof of Theorem \ref{thm:risk-bounds-strong-convex}}
To apply Theorem \ref{thm:risk-bound-mohri}, we need to control the term $\sup_{h\in H}\sum_{i=1}^n\sum_{j=1}^cg_{(j-1)n+i}h_j(x_i)$, which we tackle by the following lemma due to \cite{kakade2012regularization}.
\begin{lemma}[Corollary 4 in \cite{kakade2012regularization}]\label{lemma:kakade-fenchel}
  If $f$ is $\beta$-strongly convex w.r.t. $\|\cdot\|$ and $f^*(\bm{0})=0$, then, for any sequence $v_1,\ldots,v_n$ and for any $\mu$ we have
  $$\sum_{i=1}^n\langle v_i,\mu\rangle-f(\mu)\leq \sum_{i=1}^n\langle\triangledown f^*(v_{1:i-1},v_i)+\frac{1}{2\beta}\sum_{i=1}^n\|v_i\|^2_*,$$where $v_{1:i}$ denotes the sum $\sum_{j=1}^iv_j$.
\end{lemma}
\begin{proof}[\rm\textbf{Proof of Theorem \ref{thm:risk-bounds-strong-convex}}]
  For the hypothesis space $H$ and any $\lambda>0$, applying Lemma \ref{lemma:kakade-fenchel} with $\mu=(\bw_1,\ldots,\bw_c)$ and $v_i=\lambda(g_i\phi(x_i),g_{n+i}\phi(x_i),\ldots,g_{(c-1)n+i}\phi(x_i))$, we have
  \begin{align*}
    &\lambda\sup_{h^{\bw}\in H}\sum_{i=1}^n\sum_{j=1}^cg_{(j-1)n+i}h_j^{\bw}(x_i)=\sup_{h^{\bw}\in H}\sum_{i=1}^n\sum_{j=1}^cg_{(j-1)n+i}\langle \bw_j,\lambda\phi(x_i)\rangle\\
    &=\sup_{h^{\bw}\in H}\sum_{i=1}^n\langle(\bw_1,\ldots,\bw_c),(\lambda g_i\phi(x_i),\lambda g_{n+i}\phi(x_i),\ldots,\lambda g_{(c-1)n+i}\phi(x_i))\rangle\\
    &\leq \sup_{h^{\bw}\in H}f(\bw_1,\ldots,\bw_c)+\sum_{i=1}^n\langle\triangledown f^*(v_{1:i-1}),v_i\rangle+\frac{\lambda^2}{2\beta}\sum_{i=1}^n\|(g_i\phi(x_i),g_{n+i}\phi(x_i),\ldots,g_{(c-1)n+i}\phi(x_i))\|^2_*.
  \end{align*}
  Taking expectation on both sides w.r.t. the Gaussian variables $g_1,\ldots,g_{nc}$, the term $\sum_{i=1}^n\langle\triangledown f^*(v_{1:i-1}),v_i\rangle$ vanishes, and therefore we obtain
  $$
    \ebb_{\bm{g}}\sup_{h^{\bw}\in H}\sum_{i=1}^n\sum_{j=1}^cg_{(j-1)n+i}h_j^{\bw}(x_i)\leq\frac{\Lambda}{\lambda}+\frac{\lambda}{2\beta}\sum_{i=1}^n\ebb_{\bm{g}}\|(g_i\phi(x_i),g_{n+i}\phi(x_i),\ldots,g_{(c-1)n+i}\phi(x_i))\|^2_*.
  $$
  Choosing $\lambda=\sqrt{\frac{2\beta\Lambda}{\sum_{i=1}^n\ebb_{\bm{g}}\|(g_i\phi(x_i),g_{n+i}\phi(x_i),\ldots,g_{(c-1)n+i}\phi(x_i))\|^2_*}}$, the above inequality translates to
  $$\ebb_{\bm{g}}\sup_{h^{\bw}\in H}\sum_{i=1}^n\sum_{j=1}^cg_{(j-1)n+i}h_j^{\bw}(x_i)\leq\sqrt{\frac{2\Lambda}{\beta}\sum_{i=1}^n\ebb_{\bm{g}}\|(g_i\phi(x_i),g_{n+i}\phi(x_i),\ldots,g_{(c-1)n+i}\phi(x_i))\|^2_*}.$$
  Putting the above complexity bound into Theorem \ref{thm:risk-bound-mohri}, we obtain the stated result.
\end{proof}

\subsection{Proof on $\ell_p$-norm Multi-class Classification Generalization bounds (Corollary \ref{cor:generalization-lp-regularizer})}
The following simple lemma controls the $p$-th moment of a $N(0,1)$ distributed random variable.
We give the proof here for completeness.
\begin{lemma}\label{lem:moments}
Let $g$ be $N(0,1)$ distributed. For any $p>0$, the $p$-th moment of $g$ can be bounded by
$$[\ebb|g|^p]^{\frac{1}{p}}\leq(2p)^{\frac{1}{2}+\frac{1}{p}}.$$
\end{lemma}
\begin{proof}
   Let $\forall n\in\nbb_+:$ $\Gamma(n)=(n-1)!$ be the Gamma function. The $p$-th moment of a $N(0,1)$ distributed random variable can be exactly expressed via Gamma function~\citep{winkelbauer2012moments}:
   \begin{align*}
     \ebb|g|^p&=\frac{2^{\frac{p}{2}}}{\sqrt{\pi}}\Gamma\big(\frac{p+1}{2}\big)\leq\frac{2^{\frac{p}{2}}}{\sqrt{\pi}}\Gamma\big(\lceil\frac{p+1}{2}\rceil\big)\\
     &=\frac{2^{\frac{p}{2}}}{\sqrt{\pi}}\lceil\frac{p-1}{2}\rceil!\leq\frac{2^{\frac{p}{2}}}{\sqrt{\pi}}\sqrt{2\pi}\lceil\frac{p-1}{2}\rceil^{\lceil\frac{p-1}{2}\rceil+\frac{1}{2}}\\
     &\leq(2p)^{\frac{p}{2}+1},
   \end{align*}
   where in the above deduction we have used Stirling's approximation~\citep{robbins1955remark}:
   $$n!\leq \sqrt{2\pi}n^{n+\frac{1}{2}}e^{-n+1/(12n)}.$$
\end{proof}
\begin{proof}[\rm\textbf{Proof of Corollary} \ref{cor:generalization-lp-regularizer}]
  Let $g_1,\ldots,g_{nc}$ be independent $N(0,1)$ distributed random variables. Denote by $\tau_s=[\ebb|g_1|^s]^{\frac{1}{s}}$ the $s$th moment of a $N(0,1)$ distributed random variable. Let $q$ be any number satisfying $p\leq q\leq 2$. Introduce the function $f_q(\bw):=\frac{1}{2}\|\bw\|^2_{2,q}$. Any $h^{\bw}\in H_{q,\Lambda}$ satisfies the inequality$$f_q(\bw)=\frac{1}{2}\|\bw\|^2_{2,q}\leq\frac{1}{2}\Lambda^2.$$Since $f_q(\bw)$ is $1/q^*$-strongly convex w.r.t. the norm $\|\cdot\|_{2,q}$,
	and the dual norm of $\|\cdot\|_{2,q}$ is $\|\cdot\|_{2,q^*}$~\citep{DBLP:journals/corr/abs-0910-0610}, the summation of the squared dual norm in Theorem \ref{thm:risk-bounds-strong-convex} can be rewritten as follows:
  \begin{align*}
  \sum_{i=1}^n\ebb_{\bm{g}}\|(g_i\phi(x_i),\ldots,g_{(c-1)n+i}\phi(x_i))\|^2_{2,q^*}&=\sum_{i=1}^n\ebb_{\bm{g}}\big[\sum_{j=1}^c\|g_{(j-1)n+i}\phi(x_i)\|_2^{q^*}\big]^{\frac{2}{q^*}}\\
  &=\sum_{i=1}^n\ebb_{\bm{g}}\big[\sum_{j=1}^c|g_{(j-1)n+i}|^{q^*}\big]^{\frac{2}{q^*}}k(x_i,x_i)\\
  &\stackrel{\text{symmetry}}{=}\ebb_{\bm{g}}\big[\sum_{j=1}^c|g_j|^{q^*}\big]^{\frac{2}{q^*}}\sum_{i=1}^nk(x_i,x_i)\\
  &\stackrel{Jensen}{\leq} c^{\frac{2}{q^*}}\tau_{q^*}^2\sum_{i=1}^nk(x_i,x_i).
\end{align*}
From which Theorem \ref{thm:risk-bounds-strong-convex} immediately implies the following bounds, with probability at least $1-\delta$ and for any $h^{\bw}\in H_{q,\Lambda}$:
$$R(h^{\bw})\leq\frac{1}{n}\sum_{i=1}^n\ell(\rho_{h^{\bw}}(x_i,y_i))+\frac{4L\Lambda c^{1/q^*}\tau_{q^*}}{n}\sqrt{\frac{\pi q^*}{2}\sum_{i=1}^nk(x_i,x_i)}+3B_\ell\sqrt{\frac{\log\frac{2}{\delta}}{2n}}.$$
From the trivial inequality $\|\bw\|_{2,p}\geq\|\bw\|_{2,q}$, we immediately conclude $H_{p,\Lambda}\subset H_{q,\Lambda}$. Therefore, for any $h^{\bw}\in H_{p,\Lambda}$, we have
$$R(h^{\bw})\leq\frac{1}{n}\sum_{i=1}^n\ell(\rho_{h^{\bw}}(x_i,y_i))+\inf_{p\leq q\leq 2}\frac{4L\Lambda c^{1/q^*}\tau_{q^*}}{n}\sqrt{\frac{\pi q^*}{2}\sum_{i=1}^nk(x_i,x_i)}+3B_\ell\sqrt{\frac{\log\frac{2}{\delta}}{2n}}.$$

It can be directly checked that the function $t\to\sqrt{t}c^{1/t}$ is decreasing along the interval $(0,2\log c)$ and increasing along the interval $(2\log c,\infty)$. Therefore, the above generalization bound satisfies the inequality
\begin{multline*}
  R(h^{\bw})\leq\frac{1}{n}\sum_{i=1}^n\ell(\rho_{h^{\bw}}(x_i,y_i))+3B_\ell\sqrt{\frac{\log\frac{2}{\delta}}{2n}}+\\
  \frac{L\Lambda}{n}\sqrt{8\sum_{i=1}^nk(x_i,x_i)}\times\begin{cases}
    \sqrt{2e\log c}\tau_{2\log c},&\text{if }p\leq\frac{2\log c}{2\log c - 1},\\
    c^{\frac{p-1}{p}}\tau_{\frac{p}{p-1}}\sqrt{\frac{p}{p-1}},&\text{otherwise}.
  \end{cases}
\end{multline*}
Applying Lemma \ref{lem:moments} to bound the moments of Gaussian variables, the stated result follows immediately.
\end{proof}

\section{Proofs on the Dual Problems\label{supp:dual}}
\subsection{Equivalent Representation of $\ell_p$-norm Multi-class Classification\label{supp:micchelli}}
The equivalence between Problem \eqref{primal-problem} and Eq. \eqref{primal-problem-equivalent} follows directly from the following lemma due to \cite{micchelli2005learning}.
\begin{lemma}[\cite{micchelli2005learning}]\label{lem:micchelli}
  Let $a_i\geq0,i\in\nbb_d$ and $1\leq r<\infty$. Then
  $$\min_{\eta:\eta_i\geq0,\sum_{i\in\nbb_d}\eta_i^r\leq1}\sum_{i\in\nbb_d}\frac{a_i}{\eta_i}=\left(\sum_{i\in\nbb_d}a_i^{\frac{r}{r+1}}\right)^{1+\frac{1}{r}}$$and the minimum is attained at$$\eta_i=\frac{a_i^{\frac{1}{r+1}}}{\left(\sum_{k\in\nbb_d}a_k^{\frac{r}{r+1}}\right)^{\frac{1}{r}}}.$$
\end{lemma}
\begin{proof}[\rm\textbf{Proof of Proposition \ref{prop:mixture-update}}]
Fixing $\bw$, the sub-optimization of Eq. \eqref{primal-problem-equivalent} w.r.t. $\bm{\beta}$ is
\begin{align*}
  \min_{\bm{\beta}}&\;\sum_{j=1}^c\frac{\|\bw_j\|_2^2}{2\beta_j}\\
  \text{s.t.}&\;\|\bm{\beta}\|_{\bar{p}}\leq1,\bar{p}=p(2-p)^{-1},\beta_j\geq0.
\end{align*}
The stated result now follows directly by applying Lemma \ref{lem:micchelli} with $r=\bar{p}$ and $\alpha_j=\|\bw_j\|_2^2$.
\end{proof}
\subsection{Derivation of the Completely Dualized Problem (Problem ~\ref{prop:complete-dual})}
\begin{proof}[\rm\textbf{Derivation of Problem \ref{prop:complete-dual}}]
Problem \eqref{primal-problem} translates to the following equivalent problem
\begin{equation}
\begin{split}
  \min_{\bw}&\;\frac{1}{2}\Big[\sum_{j=1}^c\|\bw_j\|_2^p\Big]^{\frac{2}{p}}+C\sum_{i=1}^n\ell(t_i)\\
  \text{s.t.}&\;t_i\leq\langle \bw_{y_i},\phi(x_i)\rangle-\langle \bw_y,\phi(x_i)\rangle,\quad y\neq y_i,i=1,\ldots,n.
\end{split}
\end{equation}
The Lagrangian of the above convex optimization problem is
$$\mathcal{L}=\frac{1}{2}\Big[\sum_{j=1}^c\|\bw_j\|_2^p\Big]^{\frac{2}{p}}+C\sum_{i=1}^n\ell(t_i)+\sum_{i=1}^n\sum_{j\neq y_i}\tilde{\alpha}_{ij}\big(t_i+\langle \bw_j,\phi(x_i)\rangle-\langle \bw_{y_i},\phi(x_i)\rangle\big),$$with Lagrangian variables $0\leq\bm{\tilde{\alpha}}\in\rbb^{n\times(c-1)}$.
For the last term of the Lagrangian, we have the following identity:
\begin{equation}\label{dual-derivation-1}
\begin{split}
\sum_{i=1}^n\sum_{j\neq y_i}\tilde{\alpha}_{ij}\langle \bw_j-\bw_{y_i},\phi(x_i)\rangle&=\sum_{i=1}^n\sum_{j\neq y_i}\tilde{\alpha}_{ij}\langle\bw_j,\phi(x_i)\rangle-
	\sum_{i=1}^n\sum_{\tilde{j}\neq y_i}\tilde{\alpha}_{i\tilde{j}}\langle\bw_{y_i},\phi(x_i)\rangle\\
&=\sum_{j=1}^c\langle\bw_j,\sum_{i:y_i\neq j}\tilde{\alpha}_{ij}\phi(x_i)\rangle-\sum_{j=1}^c\sum_{i:y_i=j}\sum_{\tilde{j}\neq j}\tilde{\alpha}_{i\tilde{j}}\langle \bw_j,\phi(x_i)\rangle\\
&=\sum_{j=1}^c\big\langle\bw_j,\sum_{i:y_i\neq j}\tilde{\alpha}_{ij}\phi(x_i)-\sum_{i:y_i=j}\sum_{\tilde{j}\neq j}\tilde{\alpha}_{i\tilde{j}}\phi(x_i)\big\rangle.
\end{split}
\end{equation}
With this identity, the Lagrangian translates to
\begin{multline}
  \mathcal{L}=\frac{1}{2}\Big[\sum_{j=1}^c\|\bw_j\|_2^p\Big]^{\frac{2}{p}}+\sum_{j=1}^c\langle\bw_j,\sum_{i:y_i\neq j}\tilde{\alpha}_{ij}\phi(x_i)-\sum_{i:y_i=j}\sum_{\tilde{j}\neq j}
	\tilde{\alpha}_{i\tilde{j}}\phi(x_i)\rangle+\\
C\sum_{i=1}^n[\ell(t_i)+\frac{1}{C}\sum_{\tilde{j}\neq y_i}\tilde{\alpha}_{i\tilde{j}}t_i].
\end{multline}
According to the definition of Fenchel conjugate function, it holds that
\begin{equation}\label{dual-deviation-lagrangian-saddle}
\begin{split}
  \inf_{\bw,\mathbf{t}}\mathcal{L}&=-\sup_{\bw}\Big[-\frac{1}{2}\Big[\sum_{j=1}^c\|\bw_j\|_2^p\Big]^{\frac{2}{p}}-\sum_{j=1}^c\langle \bw_j,
	\sum_{i:y_i\neq j}\tilde{\alpha}_{ij}\phi(x_i)-\sum_{i:y_i=j}\sum_{\tilde{j}\neq j}\tilde{\alpha}_{i\tilde{j}}\phi(x_i)\rangle\Big]\\
    &\qquad\qquad-C\sum_{i=1}^n\sup_{t_i}[-\ell(t_i)-\sum_{j\neq y_i}\frac{1}{C}\tilde{\alpha}_{ij}t_i]\\
    &=-\Big[\frac{1}{2}\Big\|\Big(-\sum_{i:y_i\neq j}\tilde{\alpha}_{ij}\phi(x_i)+\sum_{i:y_i=j}\sum_{\tilde{j}\neq j}\tilde{\alpha}_{i\tilde{j}}\phi(x_i)\Big)_{j=1}^c\Big\|^2_{2,p}\Big]^*\\
    &\qquad-C\sum_{i=1}^n\ell^*\big(-\frac{1}{C}\sum_{j\neq y_i}\tilde{\alpha}_{ij}\big)\\
    &=-\frac{1}{2}\Big\|\Big(\sum_{i:y_i\neq j}\tilde{\alpha}_{ij}\phi(x_i)-\sum_{i:y_i=j}\sum_{\tilde{j}\neq j}\tilde{\alpha}_{i\tilde{j}}\phi(x_i)\Big)_{j=1}^c\Big\|^2_{2,\frac{p}{p-1}}
    -C\sum_{i=1}^n\ell^*\big(-\frac{1}{C}\sum_{j\neq y_i}\tilde{\alpha}_{ij}\big),
\end{split}
\end{equation}
where in the last step of the above deduction we  have used the identity: $\big(\frac{1}{2}\|\cdot\|^2\big)^*=\frac{1}{2}\|\cdot\|_*^2$ and the fact that the dual norm of $\|\cdot\|_{2,p}$ is $\|\cdot\|_{2,\frac{p}{p-1}}$. Consequently, the dual problem becomes
\begin{align*}
\sup_{\bm{\tilde{\alpha}}\in\rbb^{n\times(c-1)}}&\;-\frac{1}{2}\Big[\sum_{j=1}^c\big\|\sum_{i:y_i\neq j}\tilde{\alpha}_{ij}\phi(x_i)-\sum_{i:y_i=j}\sum_{\tilde{j}\neq j}\tilde{\alpha}_{i\tilde{j}}\phi(x_i)\big\|_2^{\frac{p}{p-1}}\Big]^{\frac{2(p-1)}{p}}
-C\sum_{i=1}^n\ell^*\big(-\frac{1}{C}\sum_{j\neq y_i}\tilde{\alpha}_{ij}\big),\\
\text{s.t.}&\;\bm{\tilde{\alpha}}\geq0.
\end{align*}
Introducing $\bm{\alpha}\in\rbb^{n\times c}$ via the substitution:
\begin{equation}\label{dual-derivation-substitution}
  \alpha_{ij}=\begin{cases}
-\tilde{\alpha}_{ij}&\text{if }j\neq y_i\\
\sum_{\tilde{j}\neq y_i}\tilde{\alpha}_{i\tilde{j}}&\text{if }j=y_i,
\end{cases}
\end{equation}
we have
\begin{equation}\label{dual-derivation-2}
  \sum_{i:y_i\neq j}\tilde{\alpha}_{ij}\phi(x_i)-\sum_{i:y_i=j}\sum_{\tilde{j}\neq j}\tilde{\alpha}_{i\tilde{j}}\phi(x_i)=-\sum_{i:y_i\neq j}\alpha_{ij}\phi(x_i)-\sum_{i:y_i=j}\alpha_{ij}\phi(x_i),
\end{equation}
from which the stated dual problem follows directly.
\end{proof}

\subsection{Proof of the Representer Theorem (Theorem \ref{thm:repre})}
Let $H_1,\ldots,H_c$ be $c$ Hilbert spaces and $p\geq1$. Define the function $g_p(v_1,\ldots,v_c):H_1\times\cdots\times H_c\to\rbb$ by $$g_p(v_1,\ldots,v_c)=\frac{1}{2}\|(v_1,\ldots,v_c)\|^2_{2,p},\quad p\geq1.$$ 
\begin{lemma}\label{lem:gradient-pnorm}
  The gradient of $g_p$ is $$\frac{\partial g_p(v_1,\ldots,v_c)}{\partial v_j}=\big[\sum_{\tilde{j}=1}^c\|v_{\tilde{j}}\|_2^p\big]^{\frac{2}{p}-1}\|v_j\|_2^{p-2}v_j.$$
\end{lemma}
\begin{proof}
By the chain rule, we have
\begin{align*}
  \frac{\partial g_p(v_1,\ldots,v_c)}{\partial v_j}&=\frac{1}{p}\big[\sum_{\tilde{j}=1}^c\|v_{\tilde{j}}\|_2^p\big]^{\frac{2}{p}-1}\frac{\partial\langle v_j,v_j\rangle^{\frac{p}{2}}}{\partial v_j}\\
  &=\frac{1}{2}\big[\sum_{\tilde{j}=1}^c\|v_{\tilde{j}}\|_2^p\big]^{\frac{2}{p}-1}\frac{\partial\langle v_j,v_j\rangle}{\partial v_j}\langle v_j,v_j\rangle^{\frac{p}{2}-1}\\
  &=\big[\sum_{\tilde{j}=1}^c\|v_{\tilde{j}}\|_2^p\big]^{\frac{2}{p}-1}\|v_j\|_2^{p-2}v_j.
\end{align*}
\end{proof}

\begin{proof}[\rm\textbf{Proof of Representer Theorem (Theorem~\ref{thm:repre})}]
In our derivation of the dual problem (see Eq. \eqref{dual-deviation-lagrangian-saddle}), the variable $\bw$ should meet the optimality in the sense that
$$\bw=\arg\max_{\bm{v}}-\frac{1}{2}\big[\sum_{j=1}^c\|\bm{v}_j\|_2^p\big]^{\frac{2}{p}}+\sum_{j=1}^c\langle \bm{v}_j,\sum_{i=1}^n\alpha_{ij}\phi(x_i)\rangle.$$ Since $(\bigtriangledown f)^{-1}=\bigtriangledown f^*$ for any convex function $f$,
and the Fenchel-conjugate of $g_p$ is $g_{p^*}$, we obtain the following representation of $\bw$:
\begin{align*}
  \bw&=\bigtriangledown g^{-1}_p\Big(\sum_{i=1}^n\alpha_{i1}\phi(x_i),\ldots,\sum_{i=1}^n\alpha_{ic}\phi(x_i)\Big)\\
  &=\bigtriangledown g_{p^*}\Big(\sum_{i=1}^n\alpha_{i1}\phi(x_i),\ldots,\sum_{i=1}^n\alpha_{ic}\phi(x_i)\Big)\\
  &=\big[\sum_{j=1}^c\|\sum_{i=1}^n\alpha_{ij}\phi(x_i)\|_2^{p^*}\big]^{\frac{2}{p^*}-1}\Big(\big\|\sum_{i=1}^n\alpha_{i1}\phi(x_i)\big\|_2^{p^*-2}\big[\sum_{i=1}^n\alpha_{i1}\phi(x_i)\big],\ldots
  \big\|\sum_{i=1}^n\alpha_{ic}\phi(x_i)\big\|_2^{p^*-2}\big[\sum_{i=1}^n\alpha_{ic}\phi(x_i)\big]\Big).
\end{align*}
That is,$$\bw_j=\big[\sum_{\tilde{j}=1}^c\|\sum_{i=1}^n\alpha_{i\tilde{j}}\phi(x_i)\|_2^{p^*}\big]^{\frac{2}{p^*}-1}\big\|\sum_{i=1}^n\alpha_{ij}\phi(x_i)\big\|_2^{p^*-2}\big[\sum_{i=1}^n\alpha_{ij}\phi(x_i)\big].$$
\end{proof}

\subsection{Derivation of Partially Dualized Problem (Problem \ref{prop:partial-dual})}
\begin{proof}[\rm\textbf{Derivation of Problem \ref{prop:partial-dual}}]
The Lagrangian of the problem \eqref{primal-problem-equivalent} w.r.t. $\bw$ is
$$\mathcal{L}=\sum_{j=1}^c\frac{\|\bw_j\|_2^2}{2\beta_j}+C\sum_{i=1}^n\ell(t_i)+\sum_{i=1}^n\sum_{j\neq y_i}\tilde{\alpha}_{ij}\big(t_i+\langle \bw_j,\phi(x_i)\rangle-\langle \bw_{y_i},\phi(x_i)\rangle\big),$$with Lagrangian variables $0\leq\bm{\tilde{\alpha}}\in\rbb^{n\times(c-1)}$.

According to the identity \eqref{dual-derivation-1}, the Lagrangian translates to
\begin{multline}
  \mathcal{L}=\sum_{j=1}^c\frac{\|\bw_j\|_2^2}{2\beta_j}+\sum_{j=1}^c\langle\bw_j,\sum_{i:y_i\neq j}\tilde{\alpha}_{ij}\phi(x_i)-\sum_{i:y_i=j}\sum_{\tilde{j}\neq j}
	\tilde{\alpha}_{i\tilde{j}}\phi(x_i)\rangle+
C\sum_{i=1}^n[\ell(t_i)+\frac{1}{C}\sum_{\tilde{j}\neq y_i}\tilde{\alpha}_{i\tilde{j}}t_i].
\end{multline}
According to the definition of Fenchel conjugate function, it holds that
\begin{align*}
  \inf_{\bw,\mathbf{t}}\mathcal{L}&=-\sum_{j=1}^c\Big[\frac{1}{\beta_j}\sup_{\bw_j}\big[-\frac{1}{2}\|\bw_j\|_2^2-\big\langle\bw_j,\beta_j\big(\sum_{i:y_i\neq j}\tilde{\alpha}_{ij}\phi(x_i)-\sum_{i:y_i=j}\sum_{\tilde{j}\neq j}\tilde{\alpha}_{i\tilde{j}}\phi(x_i)\big)\big\rangle\big]\Big]\\
    &\qquad\qquad-C\sum_{i=1}^n\sup_{t_i}[-\ell(t_i)-\sum_{j\neq y_i}\frac{1}{C}\tilde{\alpha}_{ij}t_i]\\
  &=-\sum_{j=1}^c\Big[\frac{1}{\beta_j}\Big[\frac{1}{2}\big\|\beta_j\big(\sum_{i:y_i\neq j}\tilde{\alpha}_{ij}\phi(x_i)-\sum_{i:y_i=j}\sum_{\tilde{j}\neq j}\tilde{\alpha}_{i\tilde{j}}\phi(x_i)\big)\big\|_2^2\Big]^*\Big]-C\sum_{i=1}^n\ell^*\big(-\frac{1}{C}\sum_{j\neq y_i}\tilde{\alpha}_{ij}\big)\\
    &=-\frac{1}{2}\sum_{j=1}^c\beta_j\Big\|\sum_{i:y_i\neq j}\tilde{\alpha}_{ij}\phi(x_i)-\sum_{i:y_i=j}\sum_{\tilde{j}\neq j}\tilde{\alpha}_{i\tilde{j}}\phi(x_i)\Big\|_2^2
    -C\sum_{i=1}^n\ell^*\big(-\frac{1}{C}\sum_{j\neq y_i}\tilde{\alpha}_{ij}\big),
\end{align*}
where in the last step of the above deduction we  have used the identity: $\big(\frac{1}{2}\|\cdot\|^2\big)^*=\frac{1}{2}\|\cdot\|_*^2$ and the fact that the dual norm of $\|\cdot\|_{2,2}$ is itself. Consequently, the dual problem becomes
\begin{align*}
        \sup_{\bm{\tilde{\alpha}}\in\rbb^{n\times(c-1)}}&\;-\frac{1}{2}\sum_{j=1}^c\beta_j\Big\|\sum_{i:y_i\neq j}\tilde{\alpha}_{ij}\phi(x_i)-\sum_{i:y_i=j}\sum_{\tilde{j}\neq j}\tilde{\alpha}_{i\tilde{j}}\phi(x_i)\Big\|_2^2
        -C\sum_{i=1}^n\ell^*\big(-\frac{1}{C}\sum_{j\neq y_i}\tilde{\alpha}_{ij}\big),\\
        \text{s.t.}&\;\bm{\tilde{\alpha}}\geq0.
    \end{align*}
Introducing $\bm{\alpha}\in\rbb^{n\times c}$ as in Eq. \eqref{dual-derivation-substitution} and noticing the identity \eqref{dual-derivation-2}, the above \textit{dual problem} becomes
\begin{equation}
\begin{split}
\sup_{\bm{\alpha}\in\rbb^{n\times c}}&-\frac{1}{2}\sum_{j=1}^c\beta_j\big\|\sum_{i=1}^n\alpha_{ij}\phi(x_i)\big\|_2^2-C\sum_{i=1}^n\ell^*(-\frac{\alpha_{iy_i}}{C})\\
\text{s.t.}&\sum_{j=1}^c\alpha_{ij}=0,\quad\forall i=1,2,\ldots,n,\\
&\alpha_{ij}\leq0,\qquad j\neq y_i,\forall i=1,\ldots,n.
\end{split}
\end{equation}
Note that in the above derivation of the dual problem, the variable $\bw$ should meet the optimality in the sense that
$$\bw=\arg\max_{\bm{v}}-\frac{1}{2}\sum_{j=1}^c\|\bm{v}_j\|_2^2+\sum_{j=1}^c\beta_j\langle \bm{v}_j,\sum_{i=1}^n\alpha_{ij}\phi(x_i)\rangle.$$
The representer theorem stated in Problem \ref{prop:partial-dual} follows directly from this optimization condition.
\end{proof}

\end{document}